\documentclass[11pt]{article}
\usepackage[margin=1.5in]{geometry}
\usepackage{amssymb}
\usepackage{amsmath}
\usepackage{color}
\usepackage[longend, linesnumbered]{algorithm2e}
\usepackage[utf8]{inputenc}
\usepackage{comment}
\usepackage{soul}
\usepackage{xcolor}
\usepackage{tikz}
\usetikzlibrary{positioning, shapes, arrows.meta}
\usepackage{standalone}
\usepackage{graphicx}
\usepackage{amsmath}
\usepackage{rotating}
\usepackage{pgfplots}
\usepackage{subcaption}
\usepackage{amsthm}
\pgfplotsset{compat=1.18}
\usepackage{tabularx}
\newtheorem{definition}{Definition}
\newtheorem{theorem}{Theorem}
\usepackage{float}

\usepackage{authblk}

\title{DATAMUt: Deterministic Algorithms for Time-Delay Attack Detection in Multi-Hop UAV Networks}

\author[1]{Keiwan Soltani}
\author[2]{Federico Corò}
\author[3]{Punyasha Chatterjee}
\author[1]{Sajal K. Das}

\affil[1]{{Department of CS}, {Missouri University of Science and Technology}, {USA}}

\affil[2]{Department of Computer Science and Mathematics, University of Perugia, Italy}

\affil[3]{School of Mobile Computing and Communication, Jadavpur University, India}

\begin{document}

\maketitle

\begin{abstract}
Unmanned Aerial Vehicles (UAVs), also known as drones, have gained popularity in various fields such as agriculture, emergency response, and search and rescue operations. UAV networks are susceptible to several security threats, such as wormhole, jamming, spoofing, and false data injection. Time Delay Attack (TDA) is a unique attack in which malicious UAVs intentionally delay packet forwarding, posing significant threats, especially in time-sensitive applications. It is challenging to distinguish malicious delay from benign network delay due to the dynamic nature of UAV networks, intermittent wireless connectivity, or the Store-Carry-Forward (SCF) mechanism during multi-hop communication. Some existing works propose machine learning-based centralized approaches to detect TDA, which are computationally intensive and have large message overheads. This paper proposes a novel approach \emph{DATAMUt}, where the temporal dynamics of the network are represented by a weighted time-window graph (\emph{TWiG}), and then two deterministic polynomial-time algorithms are presented to detect TDA when UAVs have global and local network knowledge. 
Simulation studies show that the proposed algorithms have reduced message overhead by a factor of five and twelve in global and local knowledge, respectively, compared to existing approaches. 
Additionally, our approaches achieve approximately 860 and 1050 times less execution time in global and local knowledge, respectively, outperforming the existing methods.
\end{abstract}

\section{Introduction}\label{sec:intro}
Unmanned Aerial Vehicles (UAVs), or drones, are remotely operated or autonomous aircraft that are used in various applications such as surveillance, agriculture, and disaster relief \cite{9899444}. 
% 8742658
UAVs are cyber-physical systems (CPS) as they integrate computation (embedded software and control systems) with physical components (UAVs flying in the physical world) \cite{Vierhauser_2018}. UAVs in particular and CPS in general require monitoring capabilities to detect and possibly mitigate erroneous and safety-critical behavior at runtime.

For some specific applications like data collection, real-time monitoring, search and rescue, surveillance, and payload delivery, a collection of UAVs is used, which collaborates to gather information through multi-hop communication and relay it to a central ground station. These UAVs construct a flexible network, generally called flying ad-hoc networks or FANETs \cite{TSAO_2022}, to expand network coverage. UAV networks \cite{SHARMA_2023} are vulnerable to various attacks such as flooding \cite{10411184}, false data injection \cite{10115584}, location spoofing \cite{9398386}, jamming \cite{9508187}, wormhole \cite{Derui2018}, blackhole \cite{George2019}, grayhole \cite{Iván2019}, including Time Delay Attack (TDA) \cite{Zhai_2023_ETD}. 
\noindent

TDAs pose a significant threat to UAV networks, particularly in time-critical applications where real-time data is essential for effective decision making. 
In these attacks, malicious UAVs intentionally delay packet forwarding, disrupting the flow of information. Suppose a UAV $A$ sends a packet to $C$ through $B$.
Normally, $B$ would immediately forward the packet when encountered with $C$. However, malicious $B$ introduces a delay $\epsilon$, which hinders timely delivery. Fig. \ref{fig:TDA} illustrates the scenario.

This subtle tactic can have significant consequences. For example, in military reconnaissance, delayed intelligence could hinder strategic planning and jeopardize mission success \cite{ZHAI_2023_HOTD}. Similarly, in search and rescue operations, delayed communication could hinder rescue efforts, as illustrated in Fig. \ref{fig:cps-scenario}. 
% UAV-based cyber-physical systems (CPS) for emergency response, as illustrated in Fig. \ref{fig:cps-scenario}, reliable and timely communication is vital. 
In this scenario, a UAV relays an emergency message about an injured person on a mountain to a distant base station (BS) via multiple intermediate UAVs. The original green path represents the quickest route to the BS. However, the presence of a malicious node imposes a delay, which diverts the message along a slower red path. This can result in critical problems, including delayed emergency response, network congestion, and reduced system reliability. Such disruptions jeopardize the effectiveness of CPS, putting those in need of urgent assistance at risk. Even in civilian applications, such as package delivery \cite{9925214, 8825836} or traffic monitoring, these attacks can cause disruptions and economic losses.

\begin{figure}[ht]
    \centering
   \resizebox{\textwidth}{!}{%
    \begin{tikzpicture}
% Drawing nodes (UAVs/Towers)
        \node at (0, 0) (uav1) {\includegraphics[width=1.2cm]{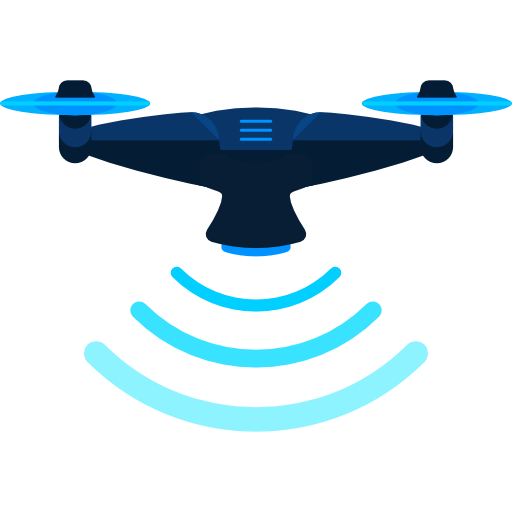}};
        \node at (2.13, -0.3) (uav2) {\includegraphics[width=1.2cm]{Figures/drone.png}};
        \node at (6.5, -1) (uav2_1) {\includegraphics[width=1.2cm]{Figures/drone.png}};
        \node at (8.8, -1.83) (uav2_2) {\includegraphics[width=1.2cm]{Figures/drone.png}};
        \node at (8.6, 0) (tower) {\includegraphics[width=1cm]{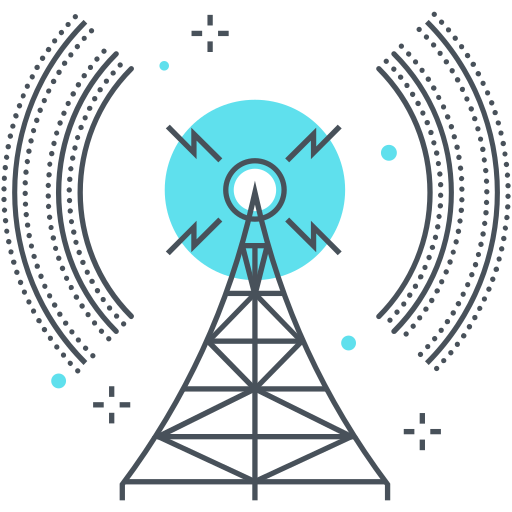}};
        
        % Drawing edges (communication links)
        \draw[red!80!black, thick, dashed] (0,0) circle(2.5cm);
        \draw[red!80!black, thick, dashed] (8.6,0) circle(2.5cm);

        \draw[->, black, thick](uav1) -- (uav2);
        \draw[->,black, very thick, dashed] (uav2) .. controls (4.5,-0.3) and (3,-1.3) .. (uav2_1);
        \draw[->, black, thick](uav2_1) -- (tower);
        \draw[->, black, very thick, dotted](uav2_1) -- (uav2_2) node[midway, above]{\LARGE $\epsilon$};
        \draw[->, black, very thick, dashed](uav2_2) -- (tower);
        %\draw[->, thick] (uav1) -- (uav2) node[midway, left] { };
        %\draw[->, thick] (uav2) -- (tower) node[midway, right] { };

        % Labels for nodes
        \node[below=0.5 cm] at (uav1) {\huge $A$};
        \node[below=0.5 cm] at (uav2) {\huge $B$};
        \node[below=0.6 cm] at (uav2_1) {\huge$B^{\prime}$};
        \node[below=0.6 cm] at (uav2_2) {\huge $B^{\prime\prime}$};
        \node[above=0.5 cm] at (tower) {\huge $C$};

    %huge
\end{tikzpicture}
}
    
    \caption{Normal and delay forwarding behaviors in the UAV network.}
    \label{fig:TDA}
\end{figure}

In addition, TDAs can act as a catalyst for other types of attacks such as omission attacks \cite{10261240}, false data injection attacks \cite{Ahmad_2020}, and data manipulation attacks \cite{Ahmad_2020}.

TDA detection is challenging in UAV networks due to the inherent characteristics, of the network such as dynamic topology, UAV mobility, and error-prone wireless communication. These can mask short delays, making it difficult to distinguish malicious behavior from normal network latency. Furthermore, the store-carry-forward (SCF) mechanism, where UAVs store and carry data until they encounter the intended recipient, can further obscure malicious delays. An attacker can exploit this mechanism to blend their delayed transmissions with legitimate delays caused by network conditions.

\begin{figure}[ht]
    \centering
    \vspace{-0.1in}
    \includegraphics[width=1\columnwidth]{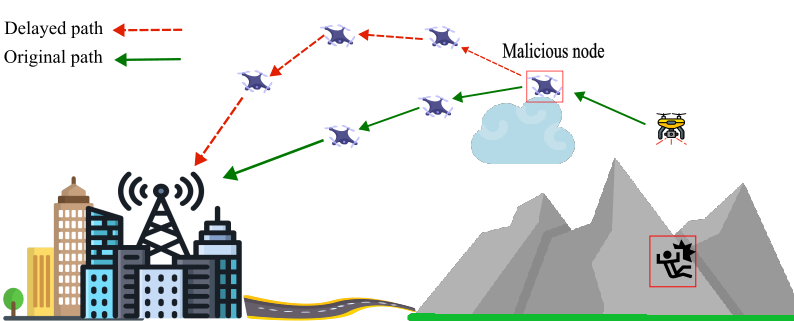}
    \caption{Effect of TDA in UAV network for emergency response}
    \label{fig:cps-scenario}
    \vspace{-0.1in}
\end{figure}

%CPS application scenario

% . By introducing delays, attackers can create opportunities for omission attacks \cite{10261240}, where packets are dropped altogether. They can also exploit the desynchronization caused by delays to inject false data or manipulate existing data, increasing the likelihood that such compromised information will be accepted as genuine.

Existing research on TDA detection in UAV networks is limited and focuses mainly on centralized approaches \cite{Zhai_2023_ETD,ZHAI_2023_HOTD}. These methods often require extensive computational resources and multiple iterations to learn the network's behavior, making them unsuitable for time-sensitive scenarios.

To address these gaps, we propose polynomial-time, deterministic algorithms, both centralized and distributed, under the name \emph{DATAMUt} (\underline{D}eterministic \underline{A}lgorithms for \underline{T}ime-Delay \underline{A}ttack Detection in \underline{M}ulti-Hop \underline{U}AV Ne\underline{t}works), for detecting TDAs in UAV networks. These algorithms have lower overhead and require only a single run to identify malicious UAVs, offering a significant improvement over existing approaches. Our contributions are as follows:
\begin{itemize}
    \item We introduce an algorithm to construct a weighted graph, named \emph{\underline{T}ime-\underline{Wi}ndow \underline{G}raph (TWiG)}, which effectively captures the temporal dynamics of UAV encounters in a multi-hop UAV network. 
    \item Based on \emph{TWiG}, we present an algorithm that leverages global knowledge of the network to identify malicious UAVs. This algorithm will serve as a benchmark for evaluating the performance of subsequent algorithms designed for more realistic scenarios.  
    \item We also propose a distributed algorithm for a more realistic scenario that operates with local knowledge, where each UAV only has information about its 2-hop neighborhood. This algorithm effectively detects TDAs, even with limited information.  
    \item Compare with the existing HOTD approach \cite{ZHAI_2023_HOTD}, we demonstrate that our approaches reduce message overhead by a factor of five and twelve in global and local knowldege, respectively. Additionally, our approaches achieve approximately 860 and 1050 times less execution time in global and local knowldge, respectively.
\end{itemize}

The rest of the paper is organized as follows: Section \ref{sec:Related-Work} reviews related works. Section \ref{sec:system_model} describes the system model. Section \ref{sec:problem_formulation} formulates the problem. Sections \ref{sec:TDA_Global} and \ref{sec:TDA-Local} present our proposed algorithms to detect malicious nodes using global and local knowledge scenarios, respectively. Section \ref{sec:Preformance-eval} evaluates the performance of our proposed methods. Finally, Section \ref{sec:Conclusion} concludes the paper.

%The remainder of this paper is organized as follows: Section \ref{sec:Related-Work} reviews related work. Section \ref{sec:Preliminaries} introduces the system model, energy model, and problem definition. Section \ref{sec:ProposedMethod} discusses the proposed method in detail, including cluster formation, UAV trajectory design, and data collection. Experimental results are given in Section \ref{sec:PerformanceEvaluation}. Finally, Section \ref{sec:Conclusion} discusses the conclusion and future work. 

\section{Related Works} \label{sec:Related-Work}
In this section, we describe the existing works on TDA in UAV networks. 
Although TDA is a well-known attack in wired networks and static wireless networks, such as CPS based smart power grids \cite{9352977, 8909732}, industrial control systems \cite{10752572}, IoT networks \cite{Zhao_2023}, cellular networks \cite{Hamici_2024}, time-sensitive networks (TSN) \cite{LUO_2023}, wireless sensor networks \cite{1542869}, and power systems \cite{9173568, 9813456}, there is a limited number of articles that specifically address TDA in UAV networks. The unique characteristics, such as 3D mobility and the SCF mechanism of UAV networks, make time delay attacks more covert, complicated, and destructive than those in other networks.

The authors in \cite{Zhai_2023_ETD} propose a time-delay attack detection framework (ETD) for UAV networks. They introduce a framework to take advantage of delays related to four dimensions: delay, node, message, and connection, while using UAV trajectory information to calculate forwarding delays. The detection model is trained using one-class classification, and malicious nodes are distinguished using the K-Means clustering algorithm based on trust values. 

The authors in \cite{ZHAI_2023_HOTD} introduce a cross-layer time-delay attack detection (HOTD) framework for UAV networks. HOTD addresses TDA by collecting delay-related features across various layers of the UAV network protocol, using supervised learning to build a model between these features and the forwarding delay. The framework utilizes clustering to distinguish between malicious and benign nodes based on their degree of consistency.

The above approaches have significant limitations. UAVs need to attach considerable information about intermediate UAVs to each transmitted message, which increases network overhead and exposes critical data. This exposure raises security risks, making the network more vulnerable to additional types of attacks, such as eavesdropping, false data injection, and replay attacks. In addition, these machine learning-based methods rely on long-term data collection, which often requires numerous iterations to accumulate sufficient information. This increases computational complexity. This complexity not only extends the detection process but also delays the identification of malicious nodes, allowing attackers to disrupt network operations before they are detected. However, these approaches are designed for centralized scenarios in which a central base station manages and processes messages, making them unsuitable for distributed UAV networks.

\section{System Model}\label{sec:system_model}

In this section, we describe the network model and TDA model, followed in this paper.

\begin{table}
  \centering
  \caption{Symbol Table}
  \label{tab:symbol}
  \begin{tabular}{|l|p{6cm}|}
    \hline
    \textbf{Symbols} & \textbf{Meaning} \\
    \hline
    $U$ & Set of UAVs\\
    \hline
    $T$ & Set of towers \\
    \hline
    $\mathcal{N}$ & Set of nodes\\
    \hline
    $n_i$ & Node id of node $i \in \mathcal{N}$\\
    \hline
    $(x_i^t, y_i^t, z_i^t)$ & coordinates of node $i \in \mathcal{N}$ at time $t$\\
     \hline
    $r_i$ & Communication range of node $i \in \mathcal{N}$\\
    \hline
     $s$ & Source node\\
    \hline
    $d$ & Destination node\\
    \hline
    $tw$ & time window\\
    \hline
    $t_{dur(n_i,n_{i+1})}$ & \emph{time window} duration between $n_i$, $n_{i+1}$\\
    \hline
    $\beta_{n_i,n_{i+1}}$ &  encounter end time between $n_i$, $n_{i+1}$\\
    \hline
    $\tau_{n_i,n_{i+1}}$ & encounter start time between $n_i$, $n_{i+1}$\\
    \hline
    $t_{n_i}^s$ & sending time of node $n_i$\\
    \hline
    $Rt_{n_i}$ & receiving time of node $n_i$\\
    \hline
    $\epsilon$ & time delay\\
    \hline
     $t_{n_i}^{s'}$ & delayed sending time of node $n_i$\\
    \hline
    $Rt_{n_i}^{'}$ & delayed receiving time of node $n_i$\\
    \hline
    $t_{tr}$ & packet transmission duration\\
    \hline
    $p_m$ & Benign routing path of packet $m$\\
    \hline
    $p'_m$ & Malicious routing path of packet $m$\\
    \hline
    $1Nb_i$ & 1-hop neighborhood of node $i$\\
    \hline
    $2Nb_i$ & 2-hop neighborhood of node $i$\\
    \hline
     $SP$ & shortest path\\
    \hline
      $FP$ & Followed path in \emph{TWiG}\\
    \hline
       $P$ & All paths\\
    \hline
    $w$ & edge weight\\
    \hline
     $\Psi$  & Malicious Nodes \\
    \hline

  \end{tabular}  
\end{table}

\subsection{Network Model}
Let us consider a 3D environment where there is a set of UAVs ($U$) and a set of towers/ground stations ($T$). UAVs can fly at different altitudes following a predefined trajectory, determined through route and task planning \cite{8413129,8618602,9408071}, and towers are fixed on the ground. 
We refer to towers and UAVs as \emph{nodes}.
So, we consider a set of nodes $\mathcal{N}=U\ \cup\ T$, where each node $i \in \mathcal{N}$ is associated with a node id $n_i$, the geographical location at any time-instant $t$, denoted as $(x_i^t, y_i^t, z_i^t)$ and a communication range $r_i$. For any node $i \in T$, $(x_i^t$, $y_i^t, z_i^t)$ is time-invariant and $z_i^t=0$. 
While flying through a trajectory, a UAV may encounter other nodes within its communication range at different time intervals/windows. 

\begin{definition}{\bf Time Window:} 
A \emph{time window} $tw$ between two nodes $n_i$ and $n_{i+1}$ is represented as $tw$=($n_i, n_{i+1}$, $\tau_{n_i,n_{i+1}}$, $\beta_{n_i,n_{i+1}})$, where $\tau_{n_i,n_{i+1}}$ and $\beta_{n_i,n_{i+1}}$ correspond to the start and end times of their encounter. The duration of $tw$ is computed as  $t_{dur(n_i,n_{i+1})} = \beta_{n_i,n_{i+1}} - \tau_{n_i,n_{i+1}}.$ 
\end{definition}

It should be noted that there might be multiple \emph{time windows} between a pair of nodes, as they may encounter more than once while traveling on their trajectories.
\begin{definition}{\bf Neighborhood:} 
Each node $n_i$ has a neighborhood $1Nb_{n_i}$, which is defined as the set of nodes it encounters while flying on its trajectory. We call it \emph{1-hop neighborhood}. Similarly, we can define the \emph{2-hop neighborhood} of node $n_i$, i.e., $2Nb_{n_i}$ = $\{1Nb_{n_i}\ \cup$ \emph{1-hop neighbors} of $1Nb_{n_i}$\}.
\end{definition}

\noindent
When a UAV wants to transmit a message to a specific ground station, direct transmission may not always be possible due to the limited communication ranges of the nodes. In that case, UAVs can communicate with neighboring nodes in specific \emph{time windows}, forming a multi-hop network to relay packets to the destination. We assume that the \emph{time windows} are sufficiently large enough to accommodate one packet transmission to support the SCF feature of a multi-hop UAV network. Let us assume that $t_{tr}$ denotes the duration to transmit a packet. Then $t_{dur} >= t_{tr}$. 

\begin{definition}{\bf Path:}
    Suppose a UAV $s_m$ (source) wants to send a packet $m$ to a tower $d_m$ (destination). The path of packet $m$ can be represented as $p_m= (s_m, n_1, n_2, \ldots, d_m)$, where $n_i$ denotes the intermediate nodes between $s_m$ and $d_m$. There can be multiple \emph{paths} between a \emph{source} and \emph{destination}. The set of all \emph{paths} between $s_m$ and $d_m$ is denoted as $P$.
\end{definition}

\begin{definition}
\textbf{Shortest Path:}
    The path that minimizes the total transmission time for delivering a packet $m$ from the source node $s_m$ to the destination node $d_m$ is called the shortest path $SP_m \in P$.
\end{definition}

For example, in Fig. \ref{fig:UAV-net}, we design a UAV network comprising seven nodes, including five UAVs $U= \{S, A, C, D, E \}$  and two towers $T=\{B, F\}$. Trajectories are drawn in red, blue, orange, green, and pink lines. $tw_1$ denotes one of the \emph{time windows} between nodes $S$ and $A$, where $\tau_{S,A}=5s$ and $\beta_{S,A}=8s$. The \emph{1-hop neighborhood} and \emph{2-hop neighborhood} of UAV $A$ are \{$S, B$\} and \{$S, B, D, C$\}, respectively. 
If node $S$ wants to send a message $m$ to node $F$, one possible \emph{path} followed by the message is $p_m = \{S, A, B, D, C, F\}$. But the \emph{shortest path} is $SP_m  = \{S, A, B, C, F\}$.

\begin{figure}[ht]
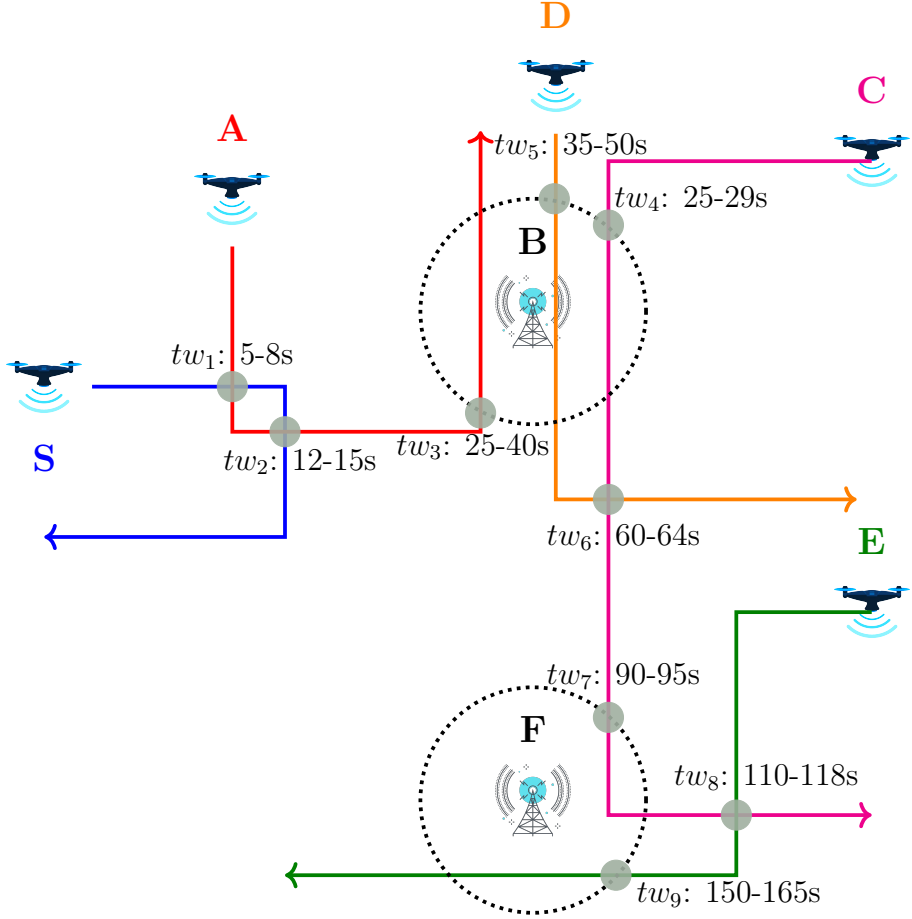

    \centering
    
    \begin{tikzpicture}

% Sender and Receiver Nodes (using your provided images)
\node at (0, 0) (sender) {\includegraphics[width=1cm]{Figures/drone.png}};
\node[below=0 cm of sender] {\Large\textbf{\textcolor{blue}{S}}};

% Drones and Towers (Nodes)
\node at (2.5, 2.5) (A) {\includegraphics[width=1cm]{Figures/drone.png}};
\node[above=0 of A] {\Large\textbf{\textcolor{red}{A}}};

\node at (6.5, 1) (B) {\includegraphics[width=1cm]{Figures/tower2.png}};
\node[above=0.cm of B] {\Large\textbf{B}};

\node at (11, 3) (C) {\includegraphics[width=1cm]{Figures/drone.png}};
\node[above=0 cm of C] {\Large\textbf{\textcolor{magenta}{C}}};

\node at (6.8, 4) (D) {\includegraphics[width=1cm]{Figures/drone.png}};
\node[above=0 cm of D] {\Large\textbf{\textcolor{orange}{D}}};

\node at (11, -3) (E) {\includegraphics[width=1cm]{Figures/drone.png}};
\node[above=0 cm of E] {\Large\textbf{\textcolor{green!50!black}{E}}};

\node at (6.5, -5.5) (F) {\includegraphics[width=1cm]{Figures/tower2.png}};
\node[above=0 cm of F] {\Large\textbf{F}}{};

%\node at (9, -4) (F) {\includegraphics[width=1cm]{tower2.png}};
%\node[below=0.3cm of F] {\textbf{F}};

% Connection paths

% Corrected Blue path for S  // (left - or right +) from current position and (Up + or down-)
\draw[line width=0.5mm, blue, ->] (sender) -- ++(3.2, 0) -- ++(0, -2) -- ++(-3.2, 0);  % Path from S to A and back

% Corrected Red path for A
%\draw[thick, red] (A) -- ++(0, -3.5) -- ++(3.3, 0) -- ++(0, 2) -- (B) -- ++(0, 1);  % A to B with detours through S
\draw[line width=0.5mm, red, ->] (A) -- ++(0, -3.1) -- ++(3.3, 0) -- ++(0, 4);  % A to B with detours through S
\draw[line width=0.5mm, orange, ->] (D) -- ++(0, -5.5) -- ++(4, 0) ;  % path D 
\draw[line width=0.5mm, magenta, ->] (C) -- ++(0, 0) -- ++(-3.5, 0) -- ++(0, -8.7) -- ++(3.5,0);  % path C
\draw[line width=0.5mm, green!50!black, ->] (E) -- ++(0, 0) -- ++(-1.8, 0) -- ++(0, -3.5) -- ++(-6,0);  % path E
% Other paths
%\draw[thick, magenta] (B) -- (C);                           % Path from B to C
%\draw[thick, orange] (B) -- (D);                            % Path from B to D
%\draw[thick, green] (C) -- (E);                             % Path from C to E
%\draw[thick, purple] (D) -- ++(0, -3) -- (F);               % Path from D to F
%\draw[thick, green] (E) -- ++(-1, -3) -- (F);               % Path from E to F

% Dotted circles around B and F
\draw[dotted, line width=0.5mm] (B) circle (1.5cm);   % Circle around B
\draw[dotted, line width=0.5mm] (F) circle (1.5cm);   % Circle around F

%\definecolor{dimgray}{RGB}{105, 105, 105}
\definecolor{dimgray}{RGB}{160, 175, 160} 
% pinpoint the intersections

\filldraw[dimgray, opacity=0.9] (2.5,0) circle (0.2cm);  % tw1
\filldraw[dimgray, opacity=0.9] (3.2, -0.6) circle (0.2cm);  % tw2
\filldraw[dimgray, opacity=0.9] (5.8, -0.35) circle (0.2cm);  % tw3
\filldraw[dimgray, opacity=0.9] (6.78, 2.5) circle (0.2cm);  % tw5
\filldraw[dimgray, opacity=0.9] (7.5, 2.15) circle (0.2cm);  % tw 4
\filldraw[dimgray, opacity=0.9] (7.5, -1.5) circle (0.2cm);  % tw 6
\filldraw[dimgray, opacity=0.9] (7.5, -4.4) circle (0.2cm);  % tw 7
\filldraw[dimgray, opacity=0.9] (9.2, -5.7) circle (0.2cm);  % tw9
\filldraw[dimgray, opacity=0.9] (7.6, -6.5) circle (0.2cm);  % tw 8

% Time labels on paths
\node at (2.5, 0.4) {\large {$tw_1$: 5-8s}};
\node at (3.4, -1) {\large{$tw_2$: 12-15s}};
\node at (5.7, -0.8) {\large{$tw_3$: 25-40s}};
\node at (8.6, 2.5) {\large{$tw_4$: 25-29s}};
\node at (7, 3.2) {\large{$tw_5$: 35-50s}};
\node at (7.7, -2) {\large{$tw_6$: 60-64s}};
\node at (7.7, -3.85) {\large{$tw_7$: 90-95s}};
\node at (9.6, -5.2) {\large{$tw_8$: 110-118s}};
\node at (9.1, -6.75) {\large{$tw_9$: 150-165s}};

\end{tikzpicture}
    
    \caption{Example of a UAV network}
    \label{fig:UAV-net}\vspace{-0.05in}
\end{figure}

\subsection{Time Delay Attack Model}

The formal description of TDA, outlined in \cite{Zhai_2023_ETD}, is as follows. In a normal scenario, a node $n_i$ transmits a packet $m$ to node $n_{i+1}$, upon encountering $n_{i+1}$, i.e., at time $\tau_{n_i,n_{i+1}}$. Hence, the time, when node $n_i$ begins to transmit the packet $m$ to node $n_{i+1}$, is denoted as $t_{n_i}^s=\tau_{n_i,n_{i+1}}$. Node $n_{i+1}$ receives the packet successfully at time $Rt_{n_{i+1}}=t_{n_i}^s+t_{tr}$, where $t_{tr}$ denotes the duration to transmit the packet.  We assume that any benign delay due to communication link instability is considered and included in $t_{tr}$.
This assumption allows us to focus on significant delays that may indicate malicious activity while accounting for the inherent communication overhead.

In a TDA, a node is malicious if it deliberately imposes a delay $\epsilon$ in transmitting the data packet to the next node. 
In addition, there may be more than one malicious node within the UAV network. In this situation, it is important to note that TDA can also change the original transmission path (followed in the normal scenario), as a node might miss the encounter \emph{time window} of the next node on the original path. For example, on a path from $S$ to $F$ in Fig. \ref{fig:UAV-net}, if node $B$ is malicious, it will hold the packet for a longer time and can miss the \emph{time window} with $C$. The possible malicious path from $S$ to $F$ could be $p'_m=\{S, A, B, D, C, F\}$, which is longer than the original path $p_m=\{S, A, B, C, F$\}. 
%\noindent
Under this scenario, the delayed sending time of the node $n_i$ and the delayed receiving time of the subsequent node $n_{i+1}$, represented as $t_{n_i}^{s\prime}$ and $Rt_{n_{i+1}}^{\prime}$, respectively, are computed as: 
\begin{equation} \label{eq:TDA_s}
   t_{n_i}^{s\prime}= \begin{cases}t_{n_i}^s + \epsilon, &  \space if  (\epsilon + t_{tr}) \leq t_{dur(n_i,n_{i+1})} \\
    \tau_{n_i,n_{i+1}}  + \epsilon, & Otherwise  \end{cases} ,
\end{equation}
\begin{equation} \label{eq:TDA_r}
   Rt_{n_{i+1}}^{\prime}=  t_{n_i}^{s\prime} + t_{tr}
\end{equation}

It should be noted that messages in a UAV network can be delayed due to other attacks too, such as position altering \cite{TSAO_2022}, wormhole \cite{Derui2018}, blackhole \cite{George2019}, or grayhole \cite{Iván2019}. However, we do not consider any combined attack in our scenario. Also, we assume that $t_{tr}$ is fixed for all messages. However, $t_{tr}$ can depend on the size of the message, the data transfer rate, and the quality of the wireless link.

The symbols used in the paper are listed in Table \ref{tab:symbol}. 

\section{Problem Formulation}\label{sec:problem_formulation}

For the fastest delivery, a packet has to traverse the shortest path (with respect to time) to the destination. The shortest path $SP_m$ that minimizes the total transmission time for delivering a packet $m$ from the source node $s_m$ to the destination node $d_m$ is:
\begin{equation} \label{eq:Spath}
   SP_m = arg \min_{p_i \in P(s_m,d_m)} \sum_{tw_j \in {Tw_p}_i} tw_j
\end{equation}

Where $P$ represents all possible paths from $s_m$ to $d_m$, and 
$Tw_{p_i}$ represents the \emph{time windows} of the \emph{path} $p_i$.
If the message passes through the \emph{shortest path}, then intermediate nodes in message transmission are benign. But if the traversed path is not the \emph{shortest path}, then a time delay attack has occurred, and some malicious nodes exist in the \emph{path}. 

\begin{theorem}\label{theorem:SP}
Given a message $m$ transmitted from the node $s_m$ to node $d_m$ along a path $p_m$, there exists at least one malicious node in $p_m$ if and only if $p_m \neq SP_m$.
\end{theorem}
\begin{proof}
Given a \emph{Path} of a message \(p_m = \{s_m,n_i, \dots, n_k, d_m\}\), where $n_i \in \mathcal N$, since the receiver cannot be the malicious node, we only consider the nodes from $n_{k}$ to $n_i$.
Assume that $n_i$ is the first node s.t.  \emph{Path}$(n_i,d_m)$ is not equal to the \emph{Shortest Path}$(n_i, d_m)$.
Note that this implies that all nodes between $n_{i+1}$ and $n_{k}$ are benign nodes.%

\noindent
Since \emph{Path}$(n_i, d_m)$ is not equal to the \emph{Shortest Path} $(n_i, d_m)$, but \emph{Path} $(n_{i+1}, d_m)$ is equal to the \emph{Shortest Path} $(n_{i+1}, d_m)$, it means that the out-degree $(n_i) \geq 2$ (i.e., there are at least two \emph{Paths} from $n_i$ to $d_m$). This implies that $n_i$ chose an edge that was not part of the \emph{Shortest Path}$(n_i, d_m)$ and is, therefore, a malicious node.
Otherwise, $n_i$ would have chosen as the next node a node belonging to the \emph{Shortest Path}$(n_i, d_m)$.
This is independent of the choices made before $n_i$, indeed even if the message was also delayed before reaching node $n_i$, the node still had a choice between the \emph{Shortest Path} and not.

So, given a \emph{Path} of a message between node $s_m \in \mathcal N$  and node $d_m \in \mathcal N$ , if the \emph{Path} of the message is not equal to the \emph{Shortest Path} between $(s_m,d_m)$, then at least one of the nodes (including $s_m$ and excluding $d_m$) in the path is malicious.s
\end{proof}

\noindent
While Theorem \ref{theorem:SP} provides a fundamental basis for the detection of malicious nodes based on the deviation from the shortest path, its practical application requires careful consideration of several factors.

First, the theorem's efficacy can be compromised by colluding malicious nodes. If multiple malicious nodes coordinate their actions and deliberately choose the same non-shortest path, they can effectively mask their malicious intent. This collusion can lead to false negatives, where malicious nodes remain undetected despite causing delays.  Developing strategies to identify such collusion is crucial for the robust detection of malicious nodes.
Second, the theorem assumes the availability of alternative paths between the source and destination. In scenarios with limited route options or the presence of cut edges in the network, the ability to identify malicious nodes definitively diminishes.  Specifically, if a node has an out-degree of 1 in the subgraph induced by the message path, it becomes impossible to ascertain whether its forwarding decision was malicious or simply constrained by the network topology.  Therefore, adapting the detection mechanism to account for these topological constraints is essential. 
Finally, the theorem operates under the assumption of global network knowledge, where nodes have complete information about network topology and time windows. However, in realistic drone networks, individual drones typically possess only local information about their neighbors.

To bridge the gap between theory and practice, we proceed in two stages. First, we present a deterministic polynomial-time algorithm for detecting malicious nodes in the idealized "Global Knowledge" scenario. 
This algorithm will serve as a benchmark for evaluating the performance of subsequent algorithms designed for more realistic scenarios.  Subsequently, we will delve into the "Local Knowledge" scenario and develop strategies to address the challenges posed by limited information and potential collusion among malicious nodes.

\section{TDA Detection using Global Knowledge}\label{sec:TDA_Global}
Here, we assume that all nodes in the network possess information about \emph{time windows} used to meet other nodes in the global network. A message contains information ($info_i$) about the nodes traversed in the path and the \emph{reception time} of the message, as illustrated in Fig. \ref{fig:message-Info-A}.  
After receiving a message, the receiver detects the malicious node(s) by following two steps: (a) Constructing a \emph{Time-Window Graph (TWiG)} from the predefined \emph{time-windows}. (b) Detecting malicious node(s) from \emph{TWiG} by comparing the shortest path ($SP_m$) and the route path ($p_m$) of a message $m$.

\begin{figure}[ht]
    \centering
\resizebox{\textwidth}{!}{%
    \begin{tikzpicture}
    % Larger main message box
    \draw[thick] (0,0) rectangle (16.6,-1);
    \node at (3,-0.5) {\huge $Message$};

    % Box for s_m
    \draw[thick] (7,-0) -- (7,-1);
    \node at (7.5,-0.5) {\Large$s_m$};

    % Box for d_m
    \draw[thick] (8,-0) -- (8,-1);
    \node at (8.5,-0.5) {\Large$d_m$};

    % Box for Info_1
    \draw[thick] (9,-0) -- (9,-1);
    \node at (9.8,-0.5) {\Large$Info_1$};d

    % Box for Info_2
    \draw[thick] (10.5,-0) -- (10.5,-1);
    \node at (11.3,-0.5) {\Large$Info_2$};

    % Box for dots
    \draw[thick] (12,-0) -- (12,-1);
    \node at (12.8,-0.5) {\Large$ \dots$};
    % Box for info i
    \draw[thick] (13.5,-0) -- (13.5,-1);
    \node at (14.3,-0.5) {\Large$Info_i$};
    % Box for info k
    \draw[thick] (15,-0) -- (15,-1);
    \node at (15.8,-0.5) {\Large$Info_k$};

    % Smaller box beneath for n_i and Rt_n_i
    \draw[thick] (9,-2) rectangle (11.2,-3);
    \node at (9.4,-2.5) {\Large $n_i$};
    \draw[thick] (10,-2) -- (10,-3);
    \node at (10.5,-2.5) {\Large$Rt_{n_i}$};

    % Curved line from Info_2 to smaller box
    \draw[line width= 0.5 mm, ->] (14.5,-1) .. controls (14.7,-1.9) and (10,-1) .. (10,-2);

\end{tikzpicture}
}
    \caption{Message and attached information.}
    \label{fig:message-Info-A}

\end{figure}
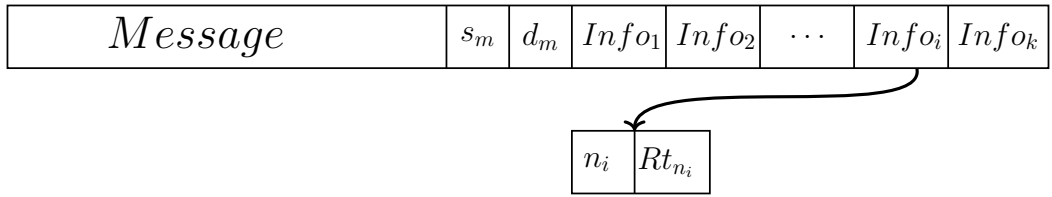

\subsection{Constructing Time-Window Graph}

From a given set of time windows $Tw=\{tw_1, \ldots, tw_n\}$, corresponding to a UAV network, Alg. \ref{alg:Full-Knowledge-Graph} constructs a weighted graph $G=(V, E, W)$ called \emph{Time-Window Graph} (\emph{TWiG}), where $V$, $E$, and $W$ are the vertex set, edge set, and weights associated with the edges.

The algorithm systematically checks all \emph{time windows} to identify potential overlaps.
To detect overlapping \emph{time windows}, the algorithm splits each of the two overlapping windows into four distinct \emph{time windows}. 
Next, two vertices are created for each \emph{time window}. 
An undirected edge is established between two vertices corresponding to a \emph{time window}, such as ($n_{x_i}^{tw}$, $n_{x_j}^{tw}$), or if a node appears in two overlapping \emph{time windows}, such as ($n_{x_i}^{tw_i}$, $n_{x_i}^{tw_j}$), or a \emph{time window} completely contained with the other ($\tau_{w_i} \leq \tau_{w_j}$ and $\beta_{tw_j} \leq \beta_{tw_i}$), as they are independent of each other. Moreover, a directed edge is formed between two vertices if the encounter end time $\beta_{tw_i}$ of a \emph{time window} is smaller than the encounter start time $\tau_{tw_j}$ of another \emph{time window}. Each edge between two nodes is assigned a weight based on the following criteria. For an undirected edge connecting two nodes from the same \emph{time window}, e.g., ($n_{a_i}^{tw}$, $n_{b_i}^{tw}$), the edge weight is set to $t_{tr}$. For other directed and undirected edges, the weights are calculated as the difference in encounter times between the two time windows ($\tau_{tw_j} - \tau_{tw_i}$). 

\textbf{Example:} 
For the UAV network in Fig. \ref{fig:UAV-net}, the constructed \emph{TWiG} is shown in Fig. \ref{fig:Full-Graph}.

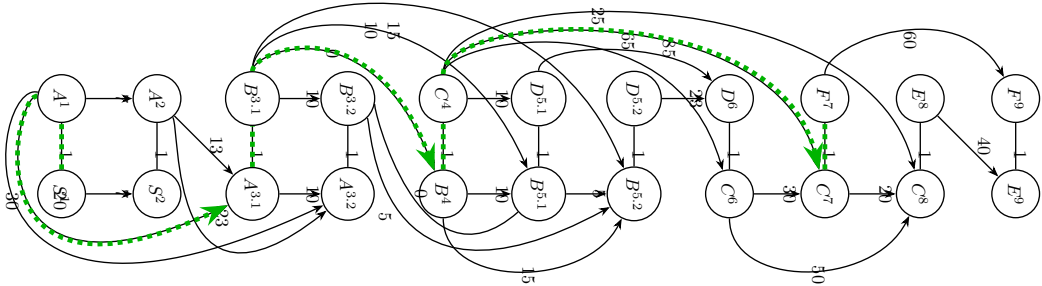
\begin{figure*}[ht]
    \centering
    \resizebox{\textwidth}{!}{%
    \begin{tikzpicture}[roundnode/.style={scale=1,circle, draw=black,  thick, minimum size=10mm}, rotate=90]
% Drawing nodes (UAVs/Towers)
        \clip (-2, -20.6) rectangle (4.1, 1.3);  %Reduce white space by clipping the bounding box
        \node[roundnode, rotate=270] (circle) at (0,0) (S1) {$S^1$};
        \node[roundnode, rotate=270] (circle) at (2,0) (A1){$A^1$};
        \node[roundnode, rotate=270] (circle) at (0,-2) (S2) {$S^2$};
        \node[roundnode, rotate=270] (circle) at (2,-2) (A2){$A^2$};
        \node[roundnode, rotate=270] (circle) at (0,-4) (A3_1){$A^{3.1}$};
        \node[roundnode, rotate=270] (circle) at (2,-4) (B3_1) {$B^{3.1}$};
        \node[roundnode, rotate=270] (circle) at (0,-6) (A3_2){$A^{3.2}$};
        \node[roundnode, rotate=270] (circle) at (2,-6) (B3_2) {$B^{3.2}$};
        \node[roundnode, rotate=270] (circle) at (0,-8) (B4) {$B^4$};
        \node[roundnode, rotate=270] (circle) at (2,-8) (C4) {$C^4$};
        \node[roundnode, rotate=270] (circle) at (0,-10) (B5_1) {$B^{5.1}$};
        \node[roundnode, rotate=270] (circle) at (2,-10) (D5_1) {$D^{5.1}$};
        \node[roundnode, rotate=270] (circle) at (0,-12) (B5_2) {$B^{5.2}$};
        \node[roundnode, rotate=270] (circle) at (2,-12) (D5_2) {$D^{5.2}$};
        \node[roundnode, rotate=270] (circle) at (0,-14) (C6) {$C^6$};
        \node[roundnode, rotate=270] (circle) at (2,-14) (D6) {$D^6$};
        \node[roundnode, rotate=270] (circle) at (0,-16) (C7) {$C^7$};
        \node[roundnode, rotate=270] (circle) at (2,-16) (F7) {$F^7$};
        \node[roundnode, rotate=270] (circle) at (0,-18) (C8) {$C^8$};
        \node[roundnode, rotate=270] (circle) at (2,-18) (E8) {$E^8$};
        \node[roundnode, rotate=270] (circle) at (0,-20) (E9) {$E^9$};
        \node[roundnode, rotate=270] (circle) at (2,-20) (F9) {$F^9$};

        % Drawing edges (communication links)

            %1- Undirected Edges
        %\draw[-, black, thick](S1) -- (A1) node[midway, above] {1} ;
        \draw[-, black, thick] (S1) -- (A1) node[midway, above, xshift = -5, yshift = -5,rotate = 270] {1};
        \draw[-, black, thick](S2) -- (A2) node[midway, above, xshift = -5, yshift = -5,rotate = 270] {1};
        \draw[-, black, thick](A3_1) -- (B3_1) node[midway, above, xshift = -5, yshift = -5,rotate = 270] {1};
        \draw[-, black, thick](A3_2) -- (B3_2) node[midway, above, xshift = -5, yshift = -5,rotate = 270] {1};
        \draw[-, black, thick](B4) -- (C4) node[midway, above, xshift = -5, yshift = -5,rotate = 270] {1};
        \draw[-, black, thick](B5_1) -- (D5_1) node[midway, above, xshift = -5, yshift = -5,rotate = 270] {1};
        \draw[-, black, thick](B5_2) -- (D5_2) node[midway, above, xshift = -5, yshift = -5,rotate = 270] {1};
        \draw[-, black, thick](C6) -- (D6) node[midway, above, xshift = -5, yshift = -5,rotate = 270] {1};
        \draw[-, black, thick](C7) -- (F7) node[midway, above, xshift = -5, yshift = -5,rotate = 270] {1};
        \draw[-, black, thick](C8) -- (E8) node[midway, above, xshift = -5, yshift = -5,rotate = 270] {1};
        \draw[-, black, thick](E9) -- (F9) node[midway, above, xshift = -5, yshift = -5,rotate = 270] {1};

        \draw[-, black, thick](B3_2) .. controls + (-0.1,-0.5) and + (-2.2,+2.5) .. (B5_1) node[midway, above, xshift = 0, yshift = -2, rotate = 270] {0};   % B3_2 to B5_1
            
            %2- Directed Straight Edges
        %\draw[->, -Stealth-, line width=1pt, black, thick] (S1) -- (S2);
        \draw[->, -Stealth, black, thick] (S1) -- (S2) node[midway, above, rotate = 270] {7};
        \draw[->, -Stealth, black, thick] (A1) -- (A2) node[midway, above, rotate = 270] {7};
        \draw[->,-Stealth, black, thick](A1) .. controls + (+0.1,+1) and + (-2.8,+6) .. (A3_1) node[midway, above, xshift = 13, yshift = -3, rotate = 270] {20};
        \draw[->,-Stealth, black, thick](A1) .. controls + (+0.4,+1.5) and + (-4,+8.5) .. (A3_2) node[midway, above, xshift =-12, yshift = 6, rotate = 270] {30};
        
        \draw[->, -Stealth, black, thick] (A2) -- (A3_1) node[midway, above, rotate = 270] {13};
        \draw[->,-Stealth, black, thick](A2) .. controls + (-0.6,-0.5) and + (-3,+3.8) .. (A3_2) node[midway, above,  xshift = 10, yshift = -2, rotate = 270] {23};

        \draw[->, -Stealth, black, thick] (A3_1) -- (A3_2) node[midway, above, rotate = 270] {10};
        \draw[->, -Stealth, black, thick] (B3_1) -- (B3_2) node[midway, above, rotate = 270] {10};
        \draw[->,-Stealth, black, thick](B3_1) .. controls + (right:1) and + (+4,1.5) .. (B4) node[midway, above,  xshift = 0, yshift = -2, rotate = 270] {0} ;
        \draw[->,-Stealth, black, thick](B3_1) .. controls + (right:1.8) and + (+4.5,+2) .. (B5_1) node[midway, above,  xshift = 0, yshift = -2, rotate = 270] {10};
        \draw[->,-Stealth, black, thick](B3_1) .. controls + (right:2.5) and + (+5,+2) .. (B5_2) node[midway, above,  xshift = -15, yshift = -15, rotate = 270] {15};

        \draw[->,-Stealth, black, thick](B3_2) .. controls + (-0.5,-0.5) and + (-3,+5.5) .. (B5_2) node[midway, above,  xshift = -17, yshift = -2, rotate = 270] {5};

        \draw[->, -Stealth, black, thick] (B4) -- (B5_1) node[midway, above , rotate = 270] {10};
        \draw[->,-Stealth, black, thick](B4) .. controls + (left:2) and + (-2, 1) .. (B5_2) node[midway, above,  xshift = -2, yshift = -2, rotate = 270] {15};

        \draw[->,-Stealth, black, thick](C4) .. controls + (right:1.2) and + (+4.2, 1.2) .. (C6) node[midway, above,  xshift = -18, yshift = 45, rotate = 270] {35};
        \draw[->,-Stealth, black, thick](C4) .. controls + (right:1.8) and + (+4.2,+1) .. (C7) node[midway, above,  xshift = 0, yshift = -5, rotate = 270] {65};
        \draw[->,-Stealth, black, thick](C4) .. controls + (right:2.5) and + (+4.5,+1.1) .. (C8) node[midway, above,  xshift = -3, yshift = -20, rotate = 270] {85};

        \draw[->,-Stealth, black, thick](D5_1) .. controls + (right:1.2) and + (+1.2, 1.2) .. (D6) node[midway, above,  xshift = -15, yshift = 20, rotate = 270] {25};
        
        \draw[->, -Stealth, black, thick] (B5_1) -- (B5_2) node[midway, above, rotate = 270] {5};
        \draw[->, -Stealth, black, thick] (D5_2) -- (D6) node[midway, above, rotate = 270] {23};
    
        \draw[->, -Stealth, black, thick] (C4) -- (D5_1) node[midway, above, rotate = 270] {10};
        \draw[->, -Stealth, black, thick] (C6) -- (C7) node[midway, above , rotate = 270] {30};
        \draw[->,-Stealth, black, thick](C6) .. controls + (left:2) and + (-2, 1) .. (C8) node[midway, above,  xshift = 0, yshift = -2, rotate = 270] {50};
        \draw[->,-Stealth, black, thick](F7) .. controls + (right:1.2) and + (+2.2, 1.2) .. (F9) node[midway, above,  xshift = 0, yshift = -5, rotate = 270] {60};
        \draw[->, -Stealth, black, thick] (C7) -- (C8) node[midway, above, rotate = 270] {20};
        \draw[->, -Stealth, black, thick] (E8) -- (E9) node[midway, above,  xshift = 2, yshift = 0, rotate = 270] {40};

        %Shortes Path From S to F
        \draw[-, green!70!black,  line width=1mm, dashed](S1) -- (A1);
        \draw[->,-Stealth, green!70!black,  line width=1mm, dashed](A1) .. controls + (+0.1,+1) and + (-2.8,+6) .. (A3_1);
        \draw[-, green!70!black,  line width=1mm, dashed](A3_1) -- (B3_1);
        \draw[->,-Stealth, green!70!black,  line width=1mm, dashed](B3_1) .. controls + (right:1) and + (+4,1.5) .. (B4);
        \draw[-, green!70!black,  line width=1mm, dashed](B4) -- (C4);
        \draw[->,-Stealth, green!70!black,  line width=1mm, dashed](C4) .. controls + (right:1.8) and + (+4.2,+1) .. (C7);
        \draw[-, green!70!black,  line width=1mm, dashed](C7) -- (F7);

\end{tikzpicture}
}
    \caption{Constructed \emph{TWiG} corresponding to the scenario in Fig. \ref{fig:UAV-net} }
    \label{fig:Full-Graph}

\end{figure*}

\RestyleAlgo{ruled}

\begin{algorithm}[ht]
\caption{\emph{Time-Window Graph} construction}\label{alg:Full-Knowledge-Graph}
\KwData{Time windows of the UAV network $Tw = \{tw_1 (n_{a_1},n_{b_1},\tau_1, \beta_1 ) , tw_2=(n_{a_2},n_{b_2},\tau_2, \beta_2),  \dots, tw_n=(n_{a_n},n_{b_n},\tau_n, \beta_n)$ \}}
\KwResult{Graph $G=(V, E, W)$}
$V=\{\} , E=\{\} $\;
\For{$tw_i= 1$ to $|Tw|$}{
    \For{$tw_j= 1$ to $|Tw|$}{
        \If{$tw_i \neq tw_j $ and $\tau_i < \tau_j$ and $\tau_j < \beta_i < \beta_j$ }{
        Replace $tw_i$ with two new time windows \;
        $tw_{i_1} = (n_{a_{i_1}}, n_{b_{i_1}}, \tau_i, \tau_j)$ \;
        $tw_{i_2} = (n_{a_{i_2}}, n_{b_{i_2}}, \tau_j, \beta_i)$ \;
        Replace $tw_j$ with two new time windows \;
        $tw_{j_1} = (n_{a_{j_1}}, n_{b_{j_1}}, \tau_j, \beta_i)$ \;
        $tw_{j_2} = (n_{a_{j_2}}, n_{b_{j_2}}, \beta_i, \beta_j)$ \;
        $tw_i =1 ; tw_j= 1$\;
        }
    }
  }
\For{$tw_i=1$ to $|Tw|$} {

Create two vertices $n_{a_i}^{tw_i}$ and $n_{b_i}^{tw_i}$ \;
$V \leftarrow V \cup n_{a_i}^{tw_i}$ \;
$V \leftarrow V \cup n_{b_i}^{tw_i}$ \;
Add an undirected edge between  $n_{a_i}^{tw_i}$ and $n_{b_i}^{tw_i}$ with $t_{tr}$ edge weight\;

$E \leftarrow E \cup (n_{a_i}^{tw_i}, n_{b_i}^{tw_i}, "U", t_{tr})$ \;
}
\For{each pair of nodes in $V: n_{a_i}^{tw_i}, n_{a_i}^{tw_j}$ }{
    \If{$\beta_{tw_i} \leq \tau_{tw_j}$}{
    Add a directed edge from $n_{a_i}^{tw_i}$ to $n_{a_i}^{tw_j}$ with $(\tau_{tw_j} - \tau_{tw_i})$ edge weight\;
    $E \leftarrow E \cup (n_{a_i}^{tw_i}, n_{a_i}^{tw_j}, "D", (\tau_{tw_j} - \tau_{tw_i}) )$\;
    }
    \If{$\tau_{tw_i} \leq \tau_{tw_j}$ and $\beta_{tw_j} \leq \beta_{tw_i} $}{
    Add an undirected edge between $n_{a_i}^{tw_i}$ to $n_{a_i}^{tw_j}$ with $(\tau_{tw_j} - \tau_{tw_i})$ edge weight\;
    
    $E \leftarrow E \cup (n_{a_i}^{tw_i}, n_{a_i}^{tw_j}, "U", (\tau_{tw_j} - \tau_{tw_i}) )$ \;
    }
}
\end{algorithm}

\vspace{0.05in}
\noindent
\textbf{Time complexity}
The \emph{TWiG} construction Algorithm \ref{alg:Full-Knowledge-Graph} has worst-case time complexity $\mathcal{O}(|Tw|^4)$. 
Specifically, checking for overlap and splitting the time windows $Tw$ takes time $\mathcal{O} 
(|Tw|^2)$, thus creating at most $|Tw|\cdot \frac{|Tw| -1}{2}$ vertices. 
The creation of vertices and undirected edges between two vertices in one time window takes $\mathcal{O}(|Tw|)$. 
Then, a directed or undirected edge is added between two nodes according to the time windows, which takes $\binom{|V|}{2} \in \mathcal{O}(|V|^2) \in \mathcal{O}(|Tw|^4)$. 

\noindent
\textbf{Space complexity:}
Since each time window contains two vertices, in the worst case, the total number of vertices is $\mathcal{O}(|Tw|^2)$, when each time window needs to be split into two. Thus, in a complete graph, where every pair of vertices is connected by an edge, the number of edges is $\mathcal{O}(|Tw|^4)$.

\vspace{-0.1in}
\subsection{Detecting Malicious Node(s)} \label{sec:detecting malicious nodes}
Algorithm \ref{alg:Detect-malicious-node} is proposed to detect malicious nodes in a global knowledge scenario. The algorithm works as follows. 

The Dijkstra algorithm is used on \emph{TWiG} to compute the shortest path ($SP$) from the source node to the destination, ensuring the minimum time required to transmit the packet. 
We prove the equivalence in Theorem \ref{thm: equivalence UAVnetwork-Twig}. 
Moreover, from the traversed path $p_m$, the receiver computes the path followed by the message $FP_m$ in \emph{TWiG}. If the total weight of the followed path matches that of the shortest path, all intermediate nodes are deemed benign, indicating that no time delay attack has occurred in that path. Otherwise, a time delay attack has occurred, and some malicious nodes exist in the path $FP_m$. 
The receiver node then examines the route followed by each intermediate node, starting from the nearest node to itself, denoted $n_i = d_m-1$ to the source node $s_m$. For each node, path $FP_{m}$ is compared to the shortest path $SP_{m}$ from that node to the destination. If the followed path does not match the shortest path, the node is identified as malicious. Upon detecting a malicious node, the evaluation of the remaining intermediate nodes preceding the malicious node is done by comparing their paths to the detected malicious node rather than to the destination. 

\noindent
\textbf{Example:} 
In Fig. \ref{fig:UAV-net}, UAV $S$ intends to send a packet $m$ to node $F$. If all intermediate nodes on the path are benign, the message is supposed to follow the path $p_m=\{S, A, B, C, F\}$. In \emph{TWiG} it is represented as $SP_m$=\{$S^1$, $A^1$, $A^{3.1}$, $B^{3.1}$, $B^4$, $C^4$, $C^7$, $F^7$\} as depicted by the green line in Fig. \ref{fig:Full-Graph}. Now, let us assume that the packet traverses the path $p'_m=\{S, A, B, D, C, F\}$, as there are some malicious nodes in the routing path. The corresponding $FP_m$ = \{ $S^1$, $A^1$, $A^{3.1}$, $A^{3.2}$, $B^{3.2}$, $B^{5.1}$, $D^{5.1}$, $D^6$, $C^6$, $C^7$, $F^7$ \} in \emph{TWiG}. 
Since $FP_m \neq SP_m$, node $F$ identifies that at least one intermediate node may be malicious. The node $F$ then systematically verifies the path followed by each node, starting from the closest node and moving back to the source node $S$. At each step, it compares the path followed with the corresponding shortest path and finds that $A$ is malicious.

\RestyleAlgo{ruled}
\begin{algorithm}[ht]
\caption{Detecting malicious nodes in global knowledge}
\label{alg:Detect-malicious-node}

\KwData{ Followed path by Message $m$ and \emph{TWiG} $G$}
\KwResult{Malicious nodes $\Psi$}
Identifying the Shortest Path $SP_m$ for $m$ by running Dijkstra's algorithm on the constructed \emph{TWiG} $G$ \;
Compute total weights ($W$) of $FP_m$  and $SP_m$\; 
Compare $W(FP_m)$ with the $W(SP_m)$ \;
$\Psi=\{\}$\;
\eIf{ $W(FP_m) == W(SP_m)$ }
{
    None of the intermediate nodes are malicious \;
    $\Psi \leftarrow  \emptyset$ \;
}{
    
    $n_i = d_m -1$ \;
    \While{ $n_i \neq s_m $}{
    
        Identifying $SP_{m}$ from $n_i$ to $d_m$  by running Dijkstra's algorithm \;
        \eIf{$FP_{m} \in SP_{m}$}{
            $n_i$ is a benign node.\;
            }{
             $n_i$ is a malicious node $\psi$ \;
             $\Psi \leftarrow \Psi \cup \psi_{n_i} $ \;
             $d_m = n_i$ \;
        }
         $n_i =  n_i - 1$ \;
    }
}

\end{algorithm}

\noindent
\textbf{Time complexity:} 
Dijkstra's algorithm has a time complexity of $\mathcal{O}(|V| + |E| \log |V|)$. Initially, the path followed is compared with the shortest path, which takes $\mathcal{O}(|V|)$ times. 
If the followed path does not match the shortest path, the shortest path is recalculated for each node in the path (within the while loop), which takes $\mathcal{O}(|V| \cdot (|V| + |E| \log |V|))$ time. Thus, the overall time complexity of Alg. \ref{alg:Detect-malicious-node} is $\mathcal{O}(|V| \cdot |E| \cdot \log |V|)$.

\begin{theorem}
\label{thm: equivalence UAVnetwork-Twig}
   For a UAV network with nodes $\mathcal{N}$, time windows $Tw$, and the corresponding Time Window Graph (TWiG) $G$, Algorithm \ref{alg:Detect-malicious-node} guarantees that any shortest path in the UAV network, with respect to total transmission time, is equivalent to a shortest path in $G$. 
\end{theorem}
\begin{proof}
   We first prove that every feasible path in the UAV network has a corresponding path in the \emph{TWiG}.
Note that any feasible path in the UAV network consists of a sequence of UAV encounters with different nodes, each corresponding to a \emph{time window}.
The \emph{TWiG}, by construction, includes vertices for every \emph{time window} and connects them with edges if the corresponding \emph{time windows} are consecutive (directly or through overlaps).
Therefore, for every pair of consecutive encounters in the UAV network, there exists a corresponding edge in the \emph{TWiG}, ensuring that any feasible path in the UAV network can be mapped to a path in the \emph{TWiG}.

Then, we prove that the path lengths (representing time) in the \emph{TWiG} accurately reflect the time taken to traverse the corresponding paths in the UAV network.
First, we note that the \emph{TWiG} assigns weights to directed edges based on the difference in encounter times, accurately representing the time it takes for the UAV to transition between encounters in the UAV network.
Then, for undirected edges, i.e., for nodes within the same \emph{time window}, the edge weight is set to the transmission time \(t_{tr}\), reflecting the time for data transmission within that encounter.
Therefore, the total length of any path in the \emph{TWiG} (i.e., the sum of edge weights) accurately represents the total time taken from the message to traverse the corresponding sequence of encounters in the UAV network, including transmission times. 
\end{proof}

\subsection*{Example on Global Knowledge:}
In this example, two malicious nodes are present, each introducing a different delay time. For example, consider nodes $A$ and $C$ as malicious nodes in the network depicted in Fig. \ref{fig:UAV-net}. The delays imposed by $A$ and $C$ are $\epsilon=5s$ and $\epsilon=6s$, respectively. In this example $s_m = S$ and $d_m = F$ of message $m$. 
Therefore, $m$ needs to traverse intermediate nodes to reach its destination. Node $S$, after generating $m$, forwards it to the first node encountered, $A$, at time $\tau=5$. Node $A$ then saves and carries the message until it encounters $B$ at time $\tau=25$. Since $A$ is a malicious node, it imposes a five-second delay and sends the packet to $B$ at $Rt_B = 31$, considering one second for the transmission time. At this point, $B$ misses the opportunity to forward $m$ to $C$, and instead sends it to $D$,  which receives the message at $Rt_D=36$.
Node $D$ saves and carries the message until it encounters $C$ at $\tau = 60$ and forwards $m$ to it. Consequently, $C$ receives $m$ at $Rt_C = 61$ and holds the packet until it meets $F$. As $C$ is also a malicious node, it introduces a delay of $\epsilon=6s$, causing it to miss the opportunity to forward $m$ to $F$, the destination. If $D$ encounters any other nodes in the network, it will impose a delay before forwarding $m$, otherwise, the message will be dropped. Since $C$ meets $E$ at $\tau =110$ and they have sufficient time to communicate (with an encounter end time $\beta = 118$), $C$ forwards $m$ to $E$, which receives it at $Rt_E = 117$, and finally $m$ reaches its destination at $Rt_F = 151$. In this scenario, identifying both malicious nodes is challenging.

To identify malicious nodes, $F$ constructs the \emph{TWiG} using Alg. \ref{alg:Full-Knowledge-Graph}. Then by Alg \ref{alg:Detect-malicious-node}, $F$ identifies malicious nodes as follows: It compares the path followed by the packet, i.e., $FP_m =\{{S^1}, {A^1}, {A^{3.1}}, {A^{3.2}}, {B^{3.2}}, {B^{5.1}}, {D^{5.1}}, {D^6},{C^6}, {C^8}, {E^8}, {E^9}, {F^9} \}$ with the shortest path of the message, i.e., $SP_m$ =$\{S^1, A^1, A^{3.1}$, $B^{3.1}, B^4, C^4, C^7, F^7\}$.
Since $FP_m \neq SP_m$, node $F$ determines that at least one of the intermediate nodes is malicious. Then, it continues to check the path followed by each node, starting from the closest node to itself and moving back toward the source node, $S$. It compares each node's followed path with the shortest path from the point where the node received the packet to itself ($F$). $FP_{E^8} =\{{E^8}, {E^9}, {F^9} \}$, $SP_{E^8} =\{{E^8}, {E^9}, {F^9} \}$ and hence $FP_{E^8} == SP_{E^8}$. This indicates that $E$ is a benign node. The next node to check in the forwarded path is $C$. $FP_{C^6} =\{{C^6},{C^8},{E^8}, {E^9}, {F^9} \}$, $SP_{C^6} =\{{C^6},{C^7},{F^7} \}$ and hence $FP_{C^6} \neq SP_{C^6}$. 
According to these results, $C$ is identified as a malicious node. For the remaining nodes in $FP_m$, the shortest path is calculated up to the identified malicious node, $C$. Next node to examine is $D$. $FP_{D^{5.1}} =\{{D^{5.1}},{D^6},{C^6}\}$, $SP_{D^{5.1}} =\{{D^{5.1}},{D^6},{C^6} \}$ and hence $FP_{D^{5.1}} == SP_{D^{5.1}}$. 
This shows that $D$ behaves normally and is not a malicious node. Next node to check is $B$. $FP_{B^{3.2}} =\{{B^{3.2}}, {B^{5.1}}, {D^{5.1}},{D^6},{C^6}\}$, $SP_{B^{3.2}} =\{{B^{3.2}}, {B^{5.1}}, {D^{5.1}},{D^6},{C^6} \}$ and hence $FP_{B^{3.2}} == SP_{B^{3.2}}$. 
The result indicates that $B$ is a benign node. The next node to check is $A$. $FP_{A^1} =\{{A^1}, {A^{3.1}}, {A^{3.2}}, {B^{3.2}}, {B^{5.1}}, {D^{5.1}},{D^6},{C^6}\}$, $SP_{A^1} =\{{A^1},{A^{3.1}} , {B^{3.1}}, {B^4}, {C^4},{C^6} \}$, $FP_{A^1} \neq SP_{A^1}$.
Based on this result, $A$ is identified as a malicious node. As the next node in $FP_m$ is the source node, the path checking is terminated here. Therefore, following the proposed method, both malicious nodes ($A$ and $C$) are successfully identified in the network in just one iteration.

\section{TDA Detection using Local Knowledge}\label{sec:TDA-Local}

In this scenario, each node has only the information about its 2-hop neighborhood, 
which includes the neighbor ID and time-windows. So, it cannot construct the \emph{TWiG} for the global network.

When a node generates a message $m$ and attempts to send it to a remote node, the source node attaches both the source ($s_m$) and destination ($d_m$) node identifiers to the message, allowing the intermediate nodes to understand its origin and destination. The packet reaches the destination by broadcasting as the nodes do not know the global network.
Intermediate nodes attach specific information to the message about each relay node, labeled $info_{i}$ (as shown in Fig. \ref{fig:message-Info-A}), which contains the current node ID ($n_i$) and the time the message was received from the previous node, $Rt_{n_{i}}$. Instead of keeping information about all intermediate nodes, only information about the two most recent intermediate nodes is kept in the message at any given time. %
Based on this information, a node follows Alg. \ref{alg:Detect-malicious-node-local} to identify malicious nodes within its \emph{2-hop neighborhood}.

\RestyleAlgo{ruled}
\begin{algorithm}[ht]
\caption{Detecting malicious nodes in Local knowledge}
\label{alg:Detect-malicious-node-local}

\KwData{Message $m$ with meta data (info), $1Nb_{n_i}$, $2Nb_{n_i}$ }
\KwResult{Malicious Nodes $\Psi$}
$\partial_{n_i,n_{i-1}} = Rt_{n_i}$ - $\tau_{n_i, n_{i-1}}$\;
$2H = 0$; \quad \textit{//Checks if the packet has traversed at least two hops}\\
\eIf{ $ \partial_{n_i, n_{i-1}} == t_{tr}$}
{
    Node $n_{i-1}$ is a benign node \;
}{
    \eIf {$(Rt_{n_{i-1}} + t_{tr}) == Rt_{n_i}$}{
        Node $n_{i-1}$ is a benign node \;
    }{
        Node $n_{i-1}$ is a malicious node $\psi$\;
        $\Psi \leftarrow \psi_{n_{i-1}}$\;
    }
}
\If{ $(info_1 \neq \emptyset)$ \textbf{and} $(info_2 \neq \emptyset)$ }{
    ${2H} = 1$;
}
\If{$2H == 1 \space$ \textbf{and} $\space n_{i-2} \in 1Nb_{n_i}$}{
    \eIf{$(Rt_{n_{i-2}} + t_{tr}) < \beta_{n_i, n_{i-2}}$ 
    \textbf{and} $ (Rt_{n_{i-1}}  + t_{tr}) > \tau_{n_i, n_{i-2}} $}{
        Node $n_{i-2}$ is a malicious node $\psi$\;
        $\Psi \leftarrow \psi_{n_{i-2}}$
    }{
        Node $n_{i-2}$ is a benign node \;
    }
}
\end{algorithm}

When a node such as $n_i$, receives a message $m$ from $n_{i-1}$, it checks whether $n_{i-1}$ delayed the message by comparing the reception time $Rt_{n_i}$ with the encounter time $\tau_{n_i,  n_{i-1}}$ and verifying it with the transmission time $t_{tr}$. 
Next, $n_i$ attempts to determine if the delay is imposed by $n_{i-1}$ or if it comes from some previous nodes. 
If it finds that $n_{i-2}$ sent the packet to $n_{i-1}$ despite having the option to send it directly to $n_i$, then $n_{i-2}$ is identified as malicious; otherwise, $n_{i-1}$ is malicious.
Once a malicious node is detected, an alert is sent to \emph{2-hop neighboring nodes} to make them aware of it.

\noindent\textbf{Example:} 
Let us consider the scenario illustrated in Fig. \ref{fig:UAV-net}, where node $S$ intends to send a message to node $F$ by relaying it through intermediate nodes. Let us assume $t_{tr}=1s$; node $A$ is malicious and imposes a 5-second delay ($\epsilon=5$). It causes the next relay node ($B$) to miss the communication time window with $C$. Consequently, $B$ sends the message to the next encountering node, $D$, which subsequently forwards the packet to $C$ before finally reaching $F$. Each node follows the steps in Alg. \ref{alg:Detect-malicious-node-local} to verify the reliability of the transmission path and identify potential malicious node(s) in its \emph{neighborhood}. 
Upon receiving the packet from $A$, node $B$ calculates $\partial_{B,A} = Rt_B$ - $\tau_{B,A} = 31s - 25s = 6s$  and since $6s\neq t_{tr}$, $A$ is marked as suspicious node. To verify, $B$ further computes $Rt_A + t_{tr} = 6s+1s = 7s$; given that $7s \neq 31s$, $A$ is deemed malicious from the perspective of $B$. 
Moreover, $B$ checks the two-hop back node, $S$, to assess its state. Since $S \notin 1Nb_{B} $, there is no need for further investigation. $B$ then appends its information to the packet and forwards it to $D$. 
$B$ also broadcasts the $A$'s maliciousness status to \emph{2-hop neighborhood}.

\noindent
\textbf{Time complexity:} 
Algorithm \ref{alg:Detect-malicious-node-local} has $O(1)$ time complexity.

\begin{theorem}
A node $n_i$ can identify a preceding node $n_{i-1}$ as malicious if any of the following conditions hold:

\begin{enumerate}
    \item Node $n_{i-1}$ forwards a message $m$ outside the expected encounter time window $\tau_{n_i, n_{i-1}}$,  where the expected reception time at $n_i$ is  $Rt_{n_i} = \tau_{n_i, n_{i-1}} + t_{tr}$ (where $t_{tr}$ is the transmission time).
    \item The reception time at node $n_i$ does not match the expected reception time based on the transmission delay from $n_{i-1}$, where $Rt_{n_i} = Rt_{n_{i-1}} + \epsilon$.
    \item A two-hop neighbor $n_{i-2}$ of $n_i$ is also a one-hop neighbor ($n_{i-2} \in 1Nb_{n_i}$) and fails to forward the message within the expected encounter window. Specifically, $n_{i-2}$ is malicious if  $(Rt_{n_{i-2}} + t_{tr}) < \beta_{n_i, n_{i-2}}$ and $Rt_{n_{i-1}} + t_{tr} > \tau_{n_i, n_{i-2}}$, where $\tau_{n_i, n_{i-2}}$ and $\beta_{n_i, n_{i-2}}$ are the start and end of encounter time between $n_i$ and $n_{i-2}$, respectively.
\end{enumerate}
\end{theorem}

\begin{proof}
    We prove each condition to identify malicious behavior:

    \begin{enumerate}
        \item 
        Let $\tau_{n_i, n_{i-1}}$ be the expected encounter time window between nodes $n_i$ and $n_{i-1}$, $t_{tr}$ be the message transmission time, and $\epsilon$ be the transmission delay. If $n_{i-1}$ is benign, the reception time $Rt_{n_i}$ at node $n_i$ should satisfy
        \(
            Rt_{n_i} = \tau_{n_i,n_{i-1}} \space + \space t_{tr} + \epsilon.
        \)
        Note that this equality can be achieved only if $\epsilon=0$.
        Therefore, if $\epsilon > 0$ we have that
        \(
        Rt_{n_i} > \tau_{n_i,n_{i-1}} \space + \space t_{tr} + \epsilon
        \)
        and then we can derive that $n_{i-1}$ is malicious.

        \item
        If $n_{i-1}$ is benign, and given the transmission delay $\epsilon$, the reception time $Rt_{n_i}$ should satisfy:
        \(
        Rt_{n_i} = Rt_{n_{i-1}} \space + \space \epsilon.
        \)
        Therefore, if 
        \(
        Rt_{n_i} \neq Rt_{n_{i-1}} \space + \space \epsilon
        \)
        then $n_{i-1}$ is malicious.

        \item
        Let $1Nb_{n_i}$ represents the set of \emph{1-hop neighbors} of $n_i$, and let $\beta_{n_i, n_{i-2}}$ be the encounter end time between nodes $n_i$ and $n_{i-2}$. If a \emph{two-hop neighbor} $n_{i-2}$ of $n_i$ is also a \emph{one-hop neighbor} ($n_{i-2} \in 1Nb_{n_i}$) and is benign, then the following condition should hold:
        \(
        Rt_{n_{i-2}} + t_{tr} \geq \beta_{n_i, n_{i-2}} \quad \text{or} \quad Rt_{n_{i-1}} + t_{tr} \leq \tau_{n_i, n_{i-2}}.
        \)
        This means that either $n_{i-2}$ forwarded the message to $n_i$ directly within the expected encounter time, or $n_{i-1}$ received the message from $n_{i-2}$ before encountering $n_i$.
        
        Therefore, if
        \(
        Rt_{n_{i-2}} + t_{tr} < \beta_{n_i, n_{i-2}} \quad \text{and} \quad Rt_{n_{i-1}} + t_{tr} > \tau_{n_i, n_{i-2}}
        \)
        then $n_{i-2}$ is malicious. This is because $n_{i-2}$ had the opportunity to send the message to $n_i$ directly but chose a 'slower' path through $n_{i-1}$.
    \end{enumerate}
\end{proof}

\noindent
{\em Limitations of Local Knowledge:}
The local knowledge scenario, while applicable to global settings, suffers from limited awareness of malicious activity. 
Unlike global knowledge, where the destination node can pinpoint the malicious node, only nodes immediately near the malicious node can identify it in the local scenario. 
Other nodes can only verify local path adherence, remaining unaware of any previous deviations if the message arrives correctly. 
This restricted awareness can be mitigated by attaching the malicious node's ID to the message or broadcasting it, but the fundamental constraint on immediate and widespread knowledge persists.

\subsection*{Example on Local Knowledge}

This example follows the same scenario as discussed in the global knowledge context. Here, there are two malicious nodes in the path, and we demonstrate how the proposed method for local knowledge in a UAV network can effectively identify these malicious nodes.
In this example, the message follows the path: $S\rightarrow A\rightarrow B\rightarrow D\rightarrow C\rightarrow E\rightarrow F$. Each node, after receiving the message, follows the steps outlined in Algorithm \ref{alg:Detect-malicious-node-local} to verify the reliability of its neighbors and identify potential malicious nodes.
The first node to receive the message is $A$, which follows the algorithm steps: $\partial_{A,S} = Rt_A$ - $\tau_{A,S} =  6s - 5s = 1s$ which is equal to $t_{tr}$ i.e. $1s$, thus confirming $S$ as benign. Since $2H \neq 1 $, the process terminates here for $A$. The next node is $B$: $\partial_{B,A} = Rt_B$ - $\tau_{B,A} = 31s - 25s = 6s  $, and since $6s\neq t_{tr}$, $A$ is considered suspicious. Node $B$ then computes $Rt_A + t_{tr}= 6s+1s=7s$ and compares it with $ Rt_B =31s$. Since $7s \neq 31s$, $A$ is identified as a malicious node from $B$'s perspective. As $S \notin 1Nb_{B}$, the process terminates for $B$ as well.

\noindent
Node $D$ is the next intermediate node. It calculates $\partial_{D,B} = Rt_D$ - $\tau_{D,B}  = 36s - 35s = 1s $, which is equal to $t_{tr}$ confirming $B$ as benign. Since $A \notin 1Nb_{D}$, the process terminates for $D$.
Next, $C$ follows the algorithm steps: $\partial_{C,D} = Rt_C$ - $\tau_{C,D} = 61s - 60s = 1s $, and since this is equal to $t_{tr}$, $D$ is confirmed to be benign. Since $B \in Nb_{C}$, $C$ makes the following comparison. If equations $(Rt_{B} + t_{tr} < \beta_{C, B})\ and\ (Rt_{A} + t_{tr} > \tau_{C, B})$ hold, $B$ is identified as malicious from $D$'s perspective. Now, $Rt_{B}=31$, $\beta_{C, B}=29$, $Rt_{A}=36$ and $\tau_{C, B}=25$. Hence, $B$ is confirmed as benign from $C$'s perspective.
Node $E$ is the next intermediate node. $\partial_{E,C} = Rt_E$ - $\tau_{E,C} = 117s - 110s = 7s$ which is not equal to $t_{tr}$, confirming that $C$ is malicious. Since $B \notin Nb_{E}$, the process terminates for $E$.
The final node, $F$, the message destination, follows the verification steps: $\partial_{F,E} = Rt_F$ - $\tau_{F,E} =151s - 150s = 1s $, which is equal to $t_{tr}$, confirming $E$ as benign. Since $C \in 1Nb_{F}$, $F$ checks if $(Rt_{C} + t_{tr} < \beta_{F, C})\ and\ (Rt_{E} + t_{tr} > \tau_{F, C})$. Since $Rt_{C}=61$, $\beta_{F, C}=95$, $Rt_{E}=117$ and $\tau_{F, C}=90$, both conditions are satisfied and $C$ is identified as malicious from the perspective of $F$'. In this example, nodes $A$ and $C$ have been successfully identified as malicious using the proposed method.

%%%%%%%%%%%%%%%%%%%%%%%%%%%%%%%%%%%%%%%%%%%%%%%%%%%%%%%%%%%%%%%%%%%%%%%%

\section{Performance Evaluation}\label{sec:Preformance-eval}

This section presents experimental results to evaluate the efficiency of \emph{DATAMUt}.
We compare the algorithms proposed for global and local knowledge with the HOTD algorithm \cite{ZHAI_2023_HOTD}, using the combination SVM+Kmeans.

\subsection{Simulation Environment}

The experiments use Python and Jupyter Notebook on a MacBook Pro equipped with an Apple M3 Pro chip, featuring a 12-core CPU, and 18 GB of RAM.
To validate the proposed algorithms, we performed experiments in five random scenarios, as shown in Table \ref{tab:scenario_table}. The simulation parameters are listed in Table \ref{tab:parameter_table}. 

\subsection{Performance Metrics}
We consider three metrics for evaluation: Network Overhead, Accuracy, and Execution Time.

\begin{table}
\centering
\begin{minipage}{0.36\linewidth} % 
    \centering
    \resizebox{\linewidth}{!}{%
    \begin{tabular}{|c|c|c|c|}
        \hline
        \textbf{Scenario} & \textbf{$\mathcal{N}$} & \textbf{$Tw$} & \textbf{$\Psi$} \\ 
        \hline
        1 & 7 & 10 & 2  \\ 
        \hline
        2 & 10 & 20 & 2  \\ 
        \hline
        3 & 15 & 30 & 3  \\ 
        \hline
        4 & 20 & 40 & 4  \\ 
        \hline
        5 & 30 & 50 & 5  \\ 
        \hline
    \end{tabular}
    }
    \caption{Scenarios}
    \label{tab:scenario_table}
\end{minipage}%
$\quad$%\hfill
\begin{minipage}{0.60\linewidth} % Adjust width if necessary
    \centering
    \resizebox{\linewidth}{!}{%
    \begin{tabular}{|c|c|}
        \hline
        \textbf{Parameters} & \textbf{Values} \\ 
        \hline
        Number of UAVs & 7 - 30   \\ 
        \hline
        Number of towers & 2- 5\\
        \hline
        UAV Speed (m/s)& $10$   \\ 
        \hline
        Message size (bit) & $1400$   \\ 
        \hline
        Node Id size (bit) & $7$   \\ 
        \hline
        Reception time size (bit) & $13$   \\ 
        \hline
        Delay duration (s) & $5$    \\ 
        \hline
        Malicious node & $0-6$ \\ 
        \hline
        Packet transmission time (s) & $1$   \\ 
        \hline
        \emph{Time-window} duration (s) & 5-15\\
        \hline
    \end{tabular}
    }
    \caption{Simulation Parameters}
    \label{tab:parameter_table}
\end{minipage}
\end{table}

\subsection{Network Overhead}\label{sec:EOR}

In UAV networks, attaching additional information to messages can lead to network overhead, impacting overall performance. 
According to \cite{Zhai_2023_ETD, ZHAI_2023_HOTD}, the extra overhead ratio (EOR) is defined as 
\begin{equation} \nonumber %\label{eq:EOR}
   EOR = \frac{\sum_{i=1}^M \sum_{j=1}^{H_i} j \times A_i}{\sum_{i=1}^M  D_i \times H_i}
\end{equation}
where $M$ denotes the total number of transmitted messages, $H_i$ represents the hop count to route $m_i$ to the destination, $D_i$ is the size of the original message $m_i$, and $A_i$ denotes the size of the additional information attached by each intermediate forwarding node to $m_i$.

In the global knowledge approach, each node appends $20$ bits of information to the main message, in contrast to the $105$ bits required by the HOTD approach. Additionally, in local knowledge approach, the proposed method adds $20$ bits per hop, resulting in a total attachment of $40$ bits to the message, as only the information for up to two hops is attached. 
Fig. \ref{fig:EOR} illustrates the extra overhead ratio introduced to the network when transmitting a message from a source to a destination through varying numbers of intermediate nodes i.e. hop counts.
For instance, in a scenario where a message traverses five intermediate nodes before reaching its destination,  the EOR for HOTD is 0.225, whereas the EOR of the proposed algorithm is only 0.042 for global knowledge and  0.028 for local knowledge. The EOR imposed by the HOTD is five times higher than that of our proposed models. Moreover, in scenarios where a message needs to traverse more nodes, the EOR imposed by the HOTD method increases substantially, while it remains constant in the local knowledge method and grows slightly in the global knowledge method. Hence, the proposed methods significantly outperform the HOTD model in terms of network overhead. 

\begin{figure}[ht]
    \centering
    
\begin{tikzpicture}
    \begin{axis}[
        width=12cm, height=8cm,
        ymin=0, ymax=0.7,
        xmin=3, xmax=15,
        ylabel={EOR},
        xlabel={Hop count},
        xtick={4, 5, 7, 10, 15},
        % Specify y-axis ticks and labels one by one
        %ytick={0.001, 0.01, 0.1, 0.5, 1, 2}, % 5, 6, 7},
        %yticklabels={0.028, 0.05, 0.1, 0.2, 0.3, 0.5, 0.6, 0.7},
        legend style={at={(0.5,1.12)}, anchor=north, legend columns=3, draw=none},
        legend cell align={left},
        mark options={solid}
    ]
    \addplot[
        color=black!70, % Corrected color specification Global Knolwedge
        mark=diamond*,
        mark options={scale=1.2},
        ultra thick
    ] coordinates {
        (4, 0.035) (5, 0.042) (7, 0.057) (10, 0.078) (15, 0.114)
    };
    \addplot[
        color=blue!70, % Corrected color specification Local Knowledge
        mark=*,
        mark options={scale=1.2},
        ultra thick
    ] coordinates {
        (4, 0.028) (5, 0.028) (7, 0.028) (10, 0.028) (15, 0.028)
    };
    \addplot[
        color=orange!120, % Corrected color specification ETD
        mark=square*,
        mark options={scale=1.2},
        ultra thick
    ] coordinates {
        (4, 0.1875) (5, 0.225) (7, 0.3) (10, 0.4125) (15, 0.6)
    };
    \legend{Global Knowledge, Local Knowledge, HOTD}
    \end{axis}
\end{tikzpicture}
    
    \caption{EOR for transmitting a message from a source to a destination with different hops in different scenarios.}
    \label{fig:EOR}
    
\end{figure}
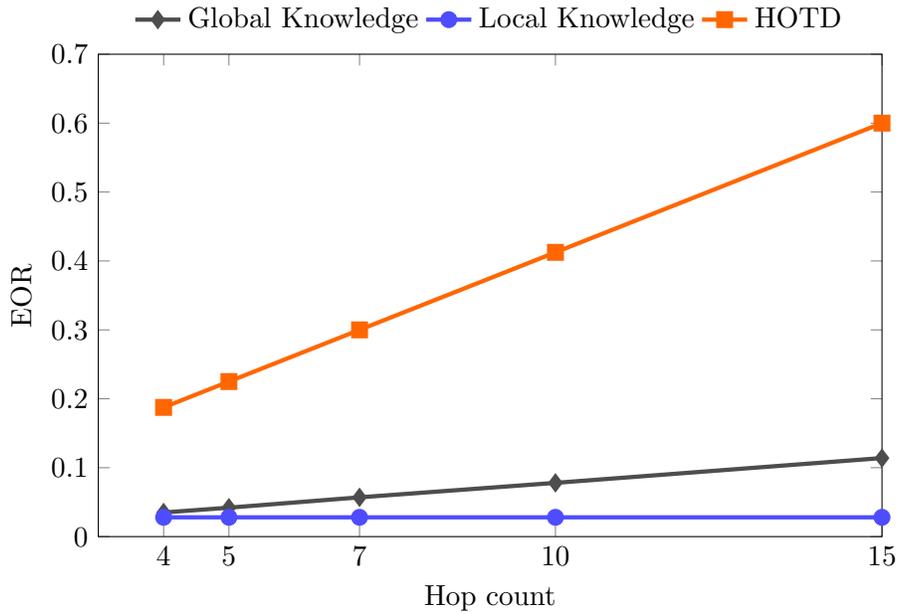

\subsection{Accuracy}
This metric illustrates the accuracy achieved by the proposed method in identifying malicious nodes within UAV networks. 
HOTD requires at least 100 and 480 iterations respectively, to attain an accuracy of around $80$\% and more than 90\%, respectively. In contrast, our proposed methods can consistently identify all existing malicious nodes within a single round in random scenarios. 

\subsection{Execution time}

In this section, we compare the execution times for local knowledge, global knowledge, and HOTD in various scenarios.
In Fig. \ref{fig:exec_time_all}, we plot the execution time for global and local knowledge with a varying number of nodes, time windows ($Tw$), and malicious nodes. 
We evaluate different time window configurations with a fixed number of nodes to demonstrate the effectiveness of the proposed algorithms in detecting malicious nodes under varying packet forwarding opportunities. 
The results indicate that increasing the number of time windows does not hinder the scalability of the proposed algorithms in terms of execution time. 
Moreover, modifying the duration of the time windows has no impact on the detection algorithms, as demonstrated by the time complexity analysis in Section \ref{sec:detecting malicious nodes}, confirming their scalability in different scenarios.

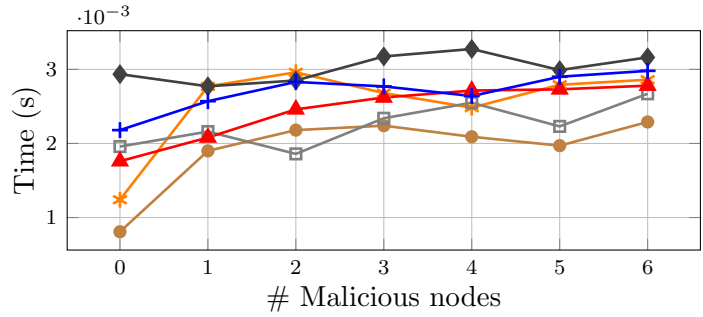
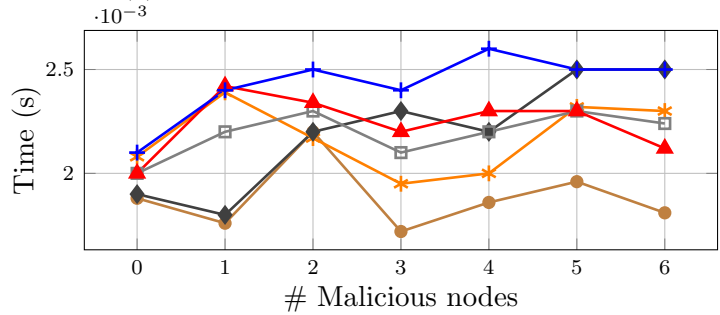
\begin{figure}[ht]
    \centering
    \begin{subfigure}[b]{0.47\textwidth}
        \centering       
\begin{tikzpicture}
		\begin{axis}[
            legend style={at={(0.45,1.15)}, anchor=south, legend columns=3, draw=none},
			every axis plot/.append style={line width=1 pt},
			xmajorgrids, ymajorgrids,
			width=10cm, height=4.5cm,
			xtick={0, 1, 2, 3, 4, 5, 6},
			every tick label/.append style={font=\scriptsize},
            xticklabel style={below},
            yticklabel style={left},
            xlabel={\# Malicious nodes}, % Label for X-axis
            ylabel={Time (s)},             % Label for Y-axis
            xlabel style={yshift= 1 mm, anchor=north}, % Adjust X-axis label position
            ylabel style={xshift= 0 mm, yshift= -2 mm, anchor=south}  % Adjust Y-axis label position
		]
		
		\addplot[brown, mark=*, error bars/.cd, y fixed,	y dir=both, y explicit] table [x index=0, y index=1] {
0 0.000808
1 0.0019
2 0.00218
3 0.00224
4 0.00209
5 0.00197
6 0.00229
% \addlegendentry{G_20_TW_30};
};  \addlegendentry{$\mathcal{N}=20 , Tw=30$};

\addplot[orange, mark=asterisk, mark options={scale=1.5}, error bars/.cd, y fixed,	y dir=both, y explicit] table [x index=0, y index=1] {
0 0.00124
1 0.00277
2 0.00296
3 0.00268
4 0.002483
5 0.002792
6 0.00286
}; \addlegendentry{$\mathcal{N}=20 , Tw=40$};

\addplot[gray!150, mark=diamond*, mark options={scale=1.5}, error bars/.cd, y fixed,	y dir=both, y explicit] table [x index=0, y index=1] {
0 0.002935
1 0.002771
2 0.00285
3 0.003174
4 0.003275
5 0.00299
6 0.003162

}; \addlegendentry{$\mathcal{N}=20 , Tw=50$};

		\addplot[gray!100, mark=square, error bars/.cd, y fixed,	y dir=both, y explicit] table [x index=0, y index=1] {
0 0.001958
1 0.00216
2 0.001859
3 0.00234
4 0.00255
5 0.00223
6 0.00267
}; \addlegendentry{$\mathcal{N}=30 , Tw=30$};

\addplot[red, mark=triangle*, mark options={scale=1.5}, error bars/.cd, y fixed,	y dir=both, y explicit] table [x index=0, y index=1] {
0 0.00176
1 0.00208
2 0.00246
3 0.00262
4 0.002714
5 0.00273
6 0.00278
}; \addlegendentry{$\mathcal{N}=30 , Tw=40$};

\addplot[blue, mark=+, mark options={scale=1.5}, error bars/.cd, y fixed,	y dir=both, y explicit] table [x index=0, y index=1] {
0 0.00218
1 0.00257
2 0.00283
3 0.00277
4 0.00264
5 0.0029
6 0.00298
}; \addlegendentry{$\mathcal{N}=30 , Tw=50$};

			\end{axis}
	\end{tikzpicture}
        
        \caption{Global Knowledge}
        \label{fig:execution_time_global}
    \end{subfigure}
    \\
    \begin{subfigure}[b]{0.47\textwidth}
        \centering
        
\begin{tikzpicture}

		\begin{axis}[
			every axis plot/.append style={line width=1 pt},
			xmajorgrids, ymajorgrids,
			width=10cm, height=4.5cm,
			xtick={0, 1, 2, 3, 4, 5, 6},
			every tick label/.append style={font=\scriptsize},
            xticklabel style={below},
            yticklabel style={left},
            xlabel={\# Malicious nodes}, % Label for X-axis
            ylabel={Time (s)},             % Label for Y-axis
            xlabel style={yshift= 1 mm, anchor=north}, % Adjust X-axis label position
            ylabel style={xshift= 0 mm, yshift= -2 mm, anchor=south}  % Adjust Y-axis label position
		]
		\addplot[brown, mark=*, error bars/.cd, y fixed,	y dir=both, y explicit] table [x index=0, y index=1] {
0 0.00188
1 0.00176
2 0.0022
3 0.00172
4 0.00186
5 0.00196
6 0.00181
};% \addlegendentry{$\mathcal{N}=20 , Tw=30$};

\addplot[orange, mark=asterisk, mark options={scale=1.5}, error bars/.cd, y fixed,	y dir=both, y explicit] table [x index=0, y index=1] {
0 0.00208
1 0.00239
2 0.00217
3 0.00195
4 0.002
5 0.00232
6 0.0023
};% \addlegendentry{$\mathcal{N}=20 , Tw=40$};

\addplot[gray!150, mark=diamond*, mark options={scale=1.5}, error bars/.cd, y fixed,	y dir=both, y explicit] table [x index=0, y index=1] {
0 0.0019
1 0.0018
2 0.0022
3 0.0023
4 0.0022
5 0.0025
6 0.0025
}; %\addlegendentry{$\mathcal{N}=20 , Tw=50$};

		\addplot[gray!100, mark=square, error bars/.cd, y fixed,	y dir=both, y explicit] table [x index=0, y index=1] {
0 0.002
1 0.0022
2 0.0023
3 0.0021
4 0.0022
5 0.0023
6 0.00224
};% \addlegendentry{$\mathcal{N}=30 , Tw=30$};

\addplot[red, mark=triangle*, mark options={scale=1.5}, error bars/.cd, y fixed,	y dir=both, y explicit] table [x index=0, y index=1] {
0 0.0020
1 0.00242
2 0.00234
3 0.0022
4 0.0023
5 0.0023
6 0.00212
}; %\addlegendentry{$\mathcal{N}=30 , Tw=40$};

\addplot[blue, mark=+, mark options={scale=1.5}, error bars/.cd, y fixed,	y dir=both, y explicit] table [x index=0, y index=1] {
0 0.0021
1 0.0024
2 0.0025
3 0.0024
4 0.0026
5 0.0025
6 0.0025
};% \addlegendentry{$\mathcal{N}=30 , Tw=50$};

			\end{axis}
	\end{tikzpicture}
        
        \caption{Local Knowledge}
        \label{fig:execution_time_local}
    \end{subfigure}
    
    \caption{Execution time of Global and Local knowledge approaches.}
    \label{fig:exec_time_all}
    \vspace{-0.1in}
\end{figure}

Furthermore, in Fig. \ref{fig:exec_time}, we compare the execution time for the global and local knowledge approaches with HOTD. The results show that the proposed approaches require considerably less time compared to HOTD, as HOTD requires multiple iterations. Specifically, the Global Knowledge approach is approximately 860 times faster, and the Local Knowledge approach is about 1050 times faster than HOTD.  Furthermore, among the two of our proposed methods, the execution time for the global knowledge approach is approximately {22\%} higher than that for the local knowledge approach, as the global knowledge method requires the construction of \emph{TWiG}.

\begin{figure}[ht]
    \centering

\begin{tikzpicture}
    \begin{axis}[
        ymode=log,
        width=12cm, height=7cm,
        ymin=0.0005, ymax=10,
        xmin=1, xmax=5,
        ylabel={Time (s)},
        xlabel={Scenario},
        xtick={1, 2, 3, 4, 5},
        % Specify y-axis ticks and labels one by one
        %ytick={0.001, 0.01, 0.1, 0.5, 1, 2}, % 5, 6, 7},
        %yticklabels={0.001, 0.01, 0.1, 0.5, 1, 2}, % 5, 6, 7},
        legend style={at={(0.5,1.12)}, anchor=north, legend columns=3, draw=none},
        legend cell align={left},
        mark options={solid}
    ]
    \addplot[
        color=black!70, % Corrected color specification
        mark=diamond*,
        mark options={scale=1.2},
        ultra thick
    ] coordinates {
        (1, 0.001236) (2, 0.001674) (3, 0.001963) (4, 0.002483) (5, 0.0029)
    };
    \addplot[
        color=blue!70, % Corrected color specification
        mark=*,
        mark options={scale=1.2},
        ultra thick
    ] coordinates {
        (1, 0.00096) (2, 0.00127) (3, 0.00166) (4, 0.002) (5, 0.0025)
    };
    \addplot[
        color=orange!120, % Corrected color specification
        mark=square*,
        mark options={scale=1.2},
        ultra thick
    ] coordinates {
        (1, 0.2032) (2, 0.5438) (3, 0.9662) (4, 1.11) (5, 5.899)
    };
    \legend{Global Knowledge, Local Knowledge, HOTD}
    \end{axis}
\end{tikzpicture}
    
    \caption{Execution time for malicious node detection in different scenarios.}
    \label{fig:exec_time}
\end{figure}
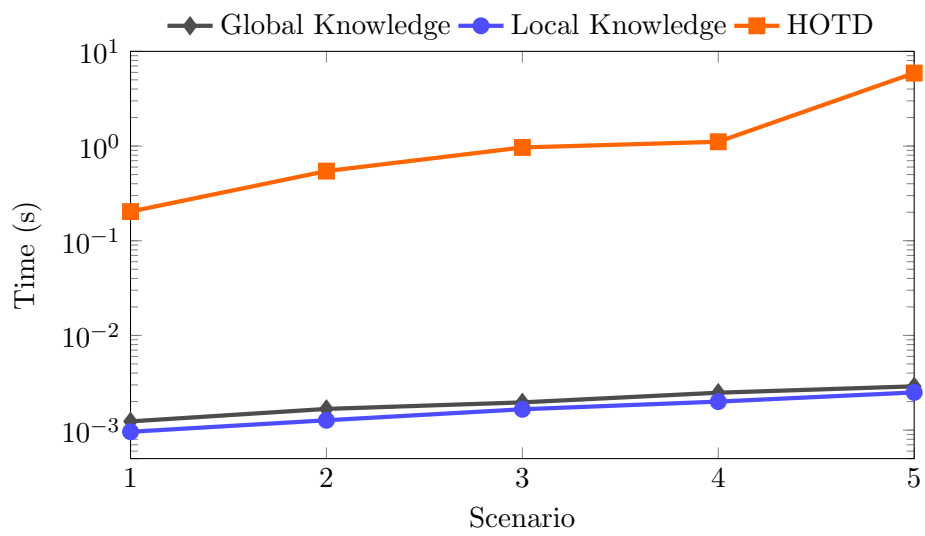

\section{Conclusion}\label{sec:Conclusion}
Time delay attacks can severely disrupt time-sensitive operations in UAV networks. This paper introduces a novel TDA detection method, \emph{DATAMUt}, applicable to both global and local knowledge scenarios. This approach achieves high detection accuracy with low computational and message overhead, enhancing the security and reliability of UAV networks.

Future research will focus on refining and extending this method by incorporating more realistic scenarios, such as variable message transmission times due to factors like network congestion, distance variations, and environmental conditions. 
Also, malicious nodes can exhibit varying behaviors depending on different scenarios and messages. A malicious node may alter the duration of the imposed delay at different times to remain anonymous. 
Additionally, we will explore advanced detection algorithms, adaptive thresholding mechanisms, and resilience strategies to mitigate the impact of detected attacks.

\clearpage
\bibliographystyle{plain} 
\bibliography{citation}
\end{document}